\documentclass[10pt,letterpaper]{article}
\usepackage[pagenumbers]{cvpr} 

\usepackage[dvipsnames]{xcolor}

\usepackage{amsthm}
\usepackage{mathtools}

\usepackage{graphicx}
\usepackage{amsmath}
\usepackage{amssymb}
\usepackage{booktabs}

\theoremstyle{definition}
\newtheorem{definition}{Definition}

\newtheorem{theorem}{Theorem}

\usepackage{amsmath,amsfonts,bm}

\def\eqref#1{equation~\ref{#1}}

\def\1{\bm{1}}

\DeclareMathAlphabet{\mathsfit}{\encodingdefault}{\sfdefault}{m}{sl}
\SetMathAlphabet{\mathsfit}{bold}{\encodingdefault}{\sfdefault}{bx}{n}

\DeclareMathOperator*{\argmax}{arg\,max}

\newtheorem{lemma}{Lemma}
\newtheorem{corollary}{Corollary}
\newtheorem{conjecture}{Conjecture}

\definecolor{cvprblue}{rgb}{0.21,0.49,0.74}
\usepackage[pagebackref,breaklinks,colorlinks,citecolor=cvprblue]{hyperref}

\usepackage{booktabs}
\usepackage{array}
\usepackage{graphicx}
\usepackage{amsmath}
\usepackage{amssymb}
\usepackage{url}
\usepackage{bm}
\usepackage{color}
\usepackage{here}
\usepackage{amsthm}
\usepackage{mathtools}
\usepackage{enumitem}
\usepackage {colortbl,array,xcolor}
\newcolumntype{M}{>{\columncolor{pink}}c}
\usepackage{float}
\usepackage{longtable}
\usepackage{wasysym}

\usepackage{caption}
\usepackage[numbers,sort&compress]{natbib}

\def\paperID{XXXXX} 
\def\confName{CVPR}
\def\confYear{2024}

\title{Geometric Insights into Focal Loss: \\Reducing Curvature for Enhanced Model Calibration}

\author{
Masanari Kimura$^*$\\
University of Melbourne\\
{\tt\small m.kimura@unimelb.edu.au}
\and
Hiroki Naganuma$^*$\\
Mila, Universit\'e de Montr\'eal, ZOZO Research \\
{\tt \small naganuma.hiroki@mila.quebec}
}

\begin{document}
\maketitle
\def\thefootnote{*}\footnotetext{These authors contributed equally to this work}\def\thefootnote{\arabic{footnote}}

\begin{abstract}
The key factor in implementing machine learning algorithms in decision-making situations is not only the accuracy of the model but also its confidence level.
The confidence level of a model in a classification problem is often given by the output vector of a softmax function for convenience.
However, these values are known to deviate significantly from the actual expected model confidence.
This problem is called model calibration and has been studied extensively.
One of the simplest techniques to tackle this task is focal loss, a generalization of cross-entropy by introducing one positive parameter.
Although many related studies exist because of the simplicity of the idea and its formalization, the theoretical analysis of its behavior is still insufficient.
In this study, our objective is to understand the behavior of focal loss by reinterpreting this function geometrically.
Our analysis suggests that focal loss reduces the curvature of the loss surface in training the model.
This indicates that curvature may be one of the essential factors in achieving model calibration.
We design numerical experiments to support this conjecture to reveal the behavior of focal loss and the relationship between calibration performance and curvature.
\end{abstract}

\section{Introduction}
\label{sec:introduction}
In recent years, neural networks have been used successfully in many fields, including computer vision~\citep{khan2018guide,zhou2012artificial}, natural language processing~\citep{goldberg2016primer,goldberg2022neural}, and signal processing~\citep{kiranyaz20191,hu2018handbook}.
Such successful results are grounded in the outstanding accuracy of neural networks.
However, for real-world applications of machine learning algorithms, merely high accuracy is not sufficient in many cases.
Especially in decision-making situations, the confidence level of the model in its own predictions is important.
For example, when the output of a machine learning model is used to assist human decision-making rather than being used directly, such as in pathology classification or credit forecasting, the focus is not only on the accuracy of the model itself but also on how much the output of the model can be trusted.
In many cases, model confidence is quantified by interpreting the output of the final layer activation function as a classification probability.
In fact, the output vector of the softmax function is treated as a probability vector because its value range is $[0,1]$.
However, although the softmax output seems to be regarded as class classification probabilities, it is known that there is a discrepancy with reality.
For example, even if a human observer is unsure whether an image is a dog or a cat, the machine learning model may determine that the image is a dog with a 99\% probability.
This phenomenon is called overconfidence~\citep{kristiadi2020being,wei2022mitigating} and is known to be one of the most important issues to be solved in practical applications of machine learning algorithms.

The task of addressing these problems is known as model calibration~\citep{guo2017calibration,mukhoti2020calibrating,wald2021calibration}.
The goal of model calibration is to ensure that the model output probabilities can be interpreted as confidence levels.
There are several metrics for evaluating model calibration, including ECE and its variants, and our objective is to improve them to guarantee the validity of the machine learning model as a confidence level for the output vector of the model.
Focal loss~\citep{lin2017focal,mukhoti2020calibrating} is one of the most popular techniques used to improve model calibration performance.
It was originally proposed as a heuristic for object detection in computer vision, but its effectiveness for model calibration was later reported.
There are many variants and related studies of focal loss due to its simplicity of conception and formulation~\citep{li2020generalized,li2021generalized,ghosh2022adafocal,li2022equalized,tao2023dual}.
The main idea of focal loss is weighting the cross-entropy according to the classification probability.
Focal loss is a generalization of cross-entropy with a positive value parameter that behaves to suppress the loss of well-classified instances.
Despite its usefulness, however, focal loss lacks theoretical analysis.
This is a drawback in the development of variant algorithms and the improvement of the model training process.

In this study, we analyze its behavior using the geometrical reinterpretation of focal loss.
In particular, we show that focal loss behaves as a reduction in the curvature of the loss surface.
This result is confirmed by both the reformulation of focal loss as an optimization under the entropy constraint and the Taylor expansion of focal loss.
Moreover, based on our analysis and the reports of existing studies that focal loss is effective for model calibration, we can expect that curvature reduction is one of the essential factors for calibration.
This conjecture suggests that regularization that explicitly reduces the curvature of the loss surface may be useful in improving model calibration performance.
Then, we design numerical experiments to support the above conjecture and investigate the behavior of focal loss.

Our contributions are summarized as follows:
\begin{itemize}
    \item We reinterpret focal loss geometrically and show that it behaves as the reduction of curvature (\cref{thm:sharpness_focal_loss}).
    \item We provide the conjecture of the connection between curvature and calibration performance (Conjecture \ref{cnj:focal_gamma_curvature} and \cref{fig:correlation}).
    \item We design and conduct the numerical experiments to demonstrate our theoretical hypothesis (\cref{fig:focal_loss_ece}).
\end{itemize}
Finally, this paper is organized as follows: we provide the preliminary knowledge and background required for our discussions in Section~\ref{sec:preliminary}, we derive the analysis of the focal loss from the perspective of geometry in Section~\ref{sec:focal_loss_and_curvature}, we design and conduct the numerical experiments in Section~\ref{sec:numerical_experiments} and we provide the conclusion and discussion in Section~\ref{sec:conclusion}.
\section{Preliminary}
\label{sec:preliminary}
In this section, we provide the preliminary knowledge and background required for our discussion.

The output of the softmax function employed in the output layer of recent neural networks for class classification is often regarded as a vector of probabilities that the input vector belongs to the respective class.
The value of this softmax output is often also interpreted as the confidence of the model with respect to the input.
However, it is known that these values differ significantly from the actual expected confidence of the models.
This phenomenon is called overconfidence~\citep{kristiadi2020being,wei2022mitigating} and is one of the major challenges to the adoption of machine learning algorithms in decision making.
For example, when the output of a machine learning model is used to assist human decision-making rather than being used directly, such as in pathology classification or credit forecasting, the focus is not only on the accuracy of the model itself but also on how much the output of the model can be trusted.

The task of ensuring that the output of a machine learning model and its confidence level are in line with the actual conditions is called model calibration~\citep{guo2017calibration,krishnan2020improving,gawlikowski2021survey}.
The goal of model calibration is to ensure that the estimated class probabilities are consistent with what would naturally occur:
\begin{align}
    \mathbb{P}\left(\argmax_{y\in\mathcal{Y}} p(y \mid \bm{x};\bm{\theta}) = y \Big| p(y \mid \bm{x};\bm{\theta}) = s\right) = s,
\end{align}
for all $s \in [0, 1]$ and $\bm{x} \in \mathcal{X}$, where $\mathcal{X}$ is an input space, $\mathcal{Y}$ is an output space and $p(\cdot \mid \cdot; \bm{\theta})$ is a model parameterized by $\bm{\theta} \in \Theta$ with a parameter space $\Theta$.
The most popular calibration performance metric is the Expected Calibration Error (ECE)~\citep{naeini2015obtaining}, which we also use in this study.
\begin{definition}[Expected Calibration Error~\citep{naeini2015obtaining}]
For $(\bm{x}, y)$, let $\hat{y} = \argmax p(y|\bm{x};\bm{\theta})$.
The expected calibration error is defined as
\begin{align}
    \mathrm{ECE}(\bm{\theta}) \coloneqq \mathbb{E}\left[\left|\mathbb{P}\left(\hat{y} = y \mid p(y|\bm{x};\bm{\theta}) = s\right) - s\right|\right],
\end{align}
where $s$ is the prediction probability.
\end{definition}
It is known that improvements in ECE can lead to achieving a confidence of the model that is in line with human intuition.

One technique known to be effective in achieving model calibration is the following focal loss~\citep{lin2017focal,mukhoti2020calibrating}.
\begin{definition}[Focal Loss~\citep{lin2017focal}]
    \label{def:focal_loss}
    For some input vector $\bm{x}\in\mathcal{X}$ and $\gamma \geq 0$, the focal loss is defined as follows.
    \begin{align}
        \mathcal{L}_{FL}(\bm{\theta};\bm{x},\gamma) \coloneqq - \sum^{|\mathcal{Y}|}_{y=1}(1 - p(y\mid \bm{x};\bm{\theta}))^\gamma \ln p(y\mid \bm{x};\bm{\theta})q(y|\bm{x}), \label{eq:focal_loss}
    \end{align}
    where $q(y\mid\bm{x})$ is the ground truth probability.
\end{definition}
Due to the simplicity of the idea and formulation of focal loss, many variants of focal loss have been proposed~\citep{li2020generalized,ghosh2022adafocal,tao2023dual}.
However, most studies on focal loss are limited to empirical reports, and theoretical analysis is lacking.
Therefore, theoretical discussions on the behavior of focal loss would be very beneficial for the derivation of new algorithms and implications for model calibration.
We can see that Eq.~\ref{eq:focal_loss} recovers cross-entropy with $\gamma=0$.
\section{Geometric Reinterpretation of Focal Loss}

\begin{figure}[htb]
    \centering	
    \includegraphics[width=0.49\linewidth]{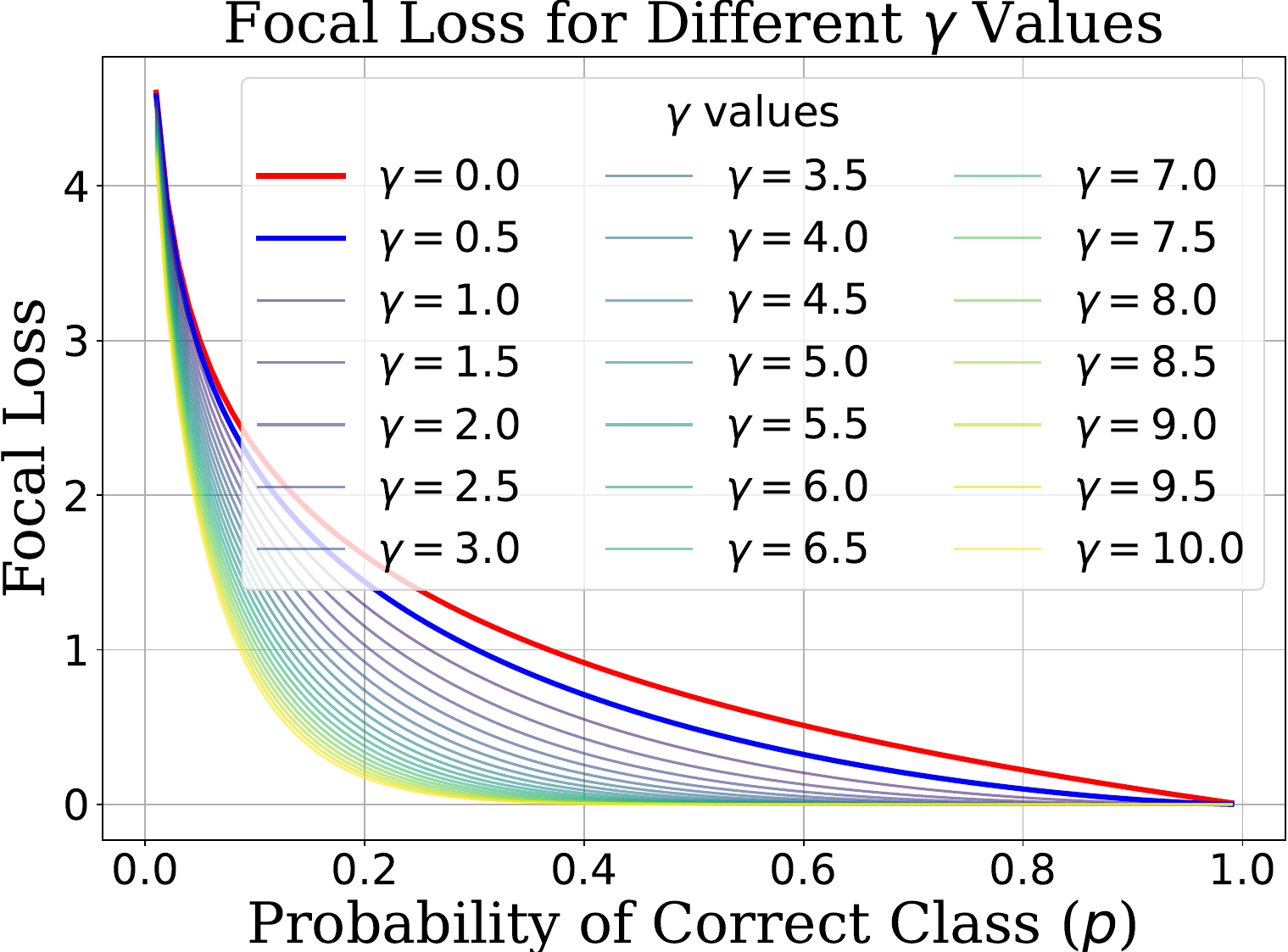}
    \includegraphics[width=0.49\linewidth]{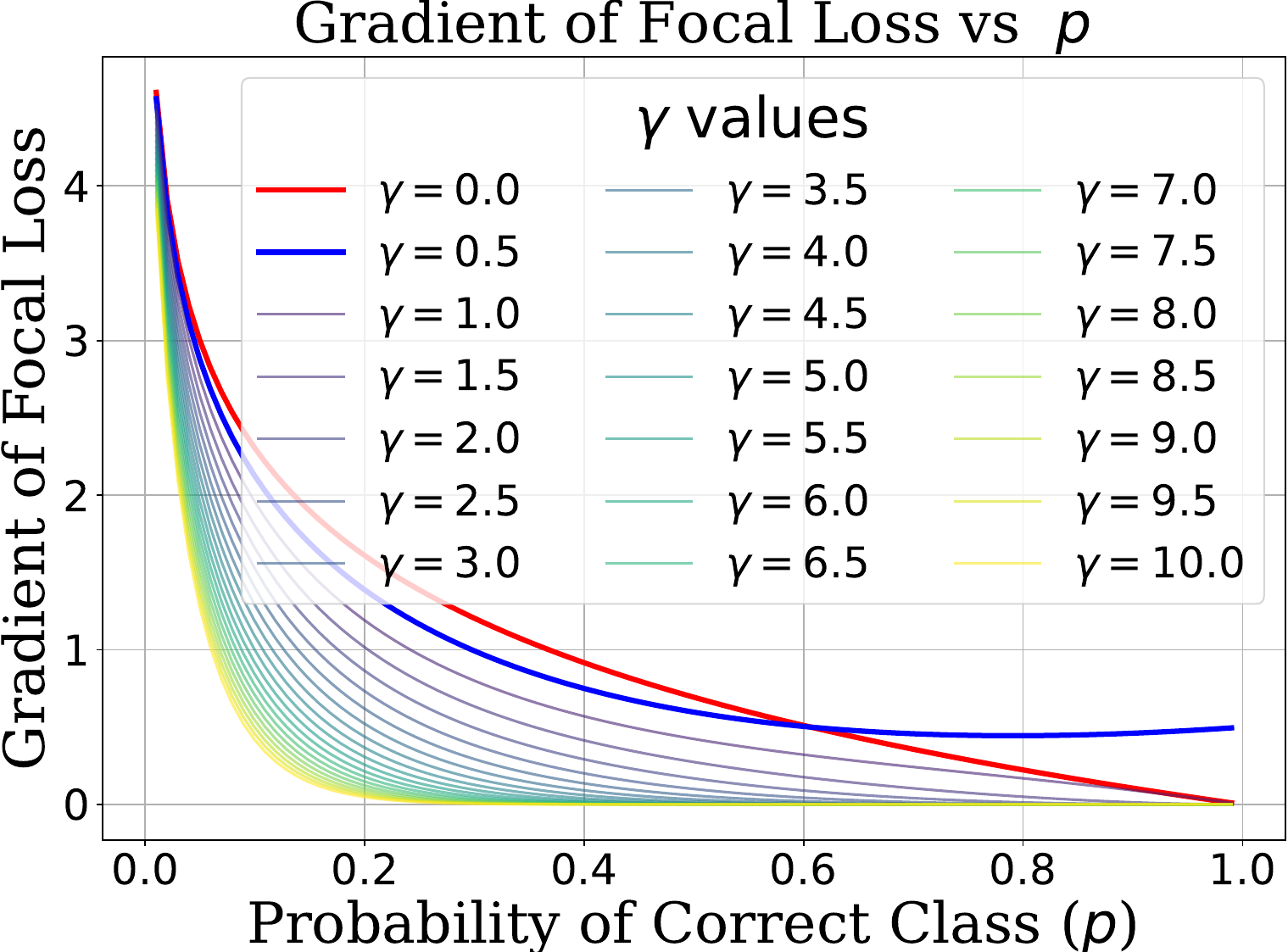}
    
    \caption{Changes in loss and gradient with respect to model prediction probability $p$ for different values of the hyperparameter $\gamma$ in focal loss. The vertical axes represent loss and gradient, respectively, and the horizontal axis represents model prediction probability $p$. Focal loss coincides with cross-entropy when $\gamma = 1$. For $0 <\gamma < 1$ , as shown in the right figure, the gradient does not converge (having an effect opposite to the original intention). When $\gamma \geq 1$, the gradient and loss for well-classified samples ($p$ close to 1) are smaller than those for cross-entropy.}
    \label{fig:focal_loss_gradient}
\end{figure}

\label{sec:focal_loss_and_curvature}
\subsection{Focal Loss as Curvature Reduction}
In this study, we reinterpret focal loss geometrically and investigate its behavior.
First, we consider the following lower bound for focal loss.
\begin{lemma}
   For $\gamma\geq 0$, we have
    \begin{align}
        \mathcal{L}_{FL}(\bm{\theta};\bm{x},\gamma) \geq \mathcal{L}_{CE}(\bm{\theta};\bm{x}) - \gamma H(y\mid\bm{x}, \bm{\theta}), \label{eq:focal_loss_and_cross_entropy}
    \end{align}
    where $\mathcal{L}_{CE}(\bm{\theta};\bm{x},\gamma) = -\sum^{|\mathcal{Y}|}_{y=1} q(y\mid\bm{x})\ln p(y\mid\bm{x};\bm{\theta})$ is the cross-entropy loss and $\mathcal{H}(y|\bm{x},\bm{\theta})$ is the conditional entropy.
\end{lemma}
\begin{proof}
    From Definition~\ref{def:focal_loss}, we have
    \begin{align}
        \mathcal{L}_{FL}(\bm{\theta};\bm{x},\gamma) &= - \sum^{|\mathcal{Y}|}_{y=1}(1 - p(y\mid \bm{x};\bm{\theta}))^\gamma \ln p(y\mid \bm{x};\bm{\theta})q(y|\bm{x}) \nonumber \\
        &\geq -\sum^{|\mathcal{Y}|}_{y=1}(1 - \gamma p(y\mid\bm{x};\bm{\theta}))\ln p(y\mid \bm{x};\bm{\theta})q(y|\bm{x}) \nonumber \\
        &\geq -\sum^{|\mathcal{Y}|}_{y=1}q(y\mid\bm{x})\ln p(y\mid \bm{x}; \bm{\theta}) \nonumber \\
        &\quad - \gamma \max_{k}q(k \mid \bm{x})\sum^{|\mathcal{Y}|}_{y=1}\left|p(y\mid\bm{x};\bm{\theta})\ln p(y\mid\bm{x};\bm{\theta})\right| \nonumber \\
        &\geq \mathcal{L}_{CE}(\bm{\theta}; \bm{x}) - \gamma H(y \mid \bm{x}; \bm{\theta}).
    \end{align}
    Here, the first inequality follows from Bernoulli's inequality and the second from H\"{o}lder's inequality.
\end{proof}
From Eq.~\ref{eq:focal_loss_and_cross_entropy}, learning procedure with focal loss can be regarded as maximizing conditional entropy term under the constraint
\begin{align}
    \int_{\Theta} p(\bm{\theta})\mathcal{L}_{CE}(\bm{\theta};\bm{x})d\bm{\theta} = \delta_\gamma \leq \delta, \label{eq:focal_loss_constraint}
\end{align}
for some $\delta\geq 0$ and $\delta_\gamma\geq 0$.
\begin{theorem}
\label{thm:maxwell_boltzmann}
Among all distributions defined on $\Theta$ with a given $\delta_\gamma$, the distribution with the largest entropy is the Maxwell-Boltzmann distribution
\begin{align}
    \tilde{p}(\bm{\theta}) &= \alpha\cdot e^{-\beta\cdot\mathcal{L}_{CE}(\bm{\theta})} \nonumber\\
    &= \alpha\cdot\exp\left\{\beta\int_{\mathcal{X}}\sum^{|\mathcal{Y}|}_{y=1}q(y|\bm{x})\ln p(y|\bm{x};\bm{\theta})d\bm{x}\right\}, \label{eq:maxwell_boltzmann}
\end{align}
where the constants $\alpha$ and $\beta$ are determined from the following constraints
\begin{align}
    \int_\Theta p(\bm{\theta})d\bm{\theta} = 1 &\Longleftrightarrow \alpha = \left(\int_\Theta e^{-\beta\cdot\mathcal{L}_{CE}(\bm{\theta})}d\bm{\theta}\right)^{-1}, \label{eq:maxwell_boltzmann_const_1} \\
    \int_\Theta\mathcal{L}_{CE}(\bm{\theta})p(\bm{\theta})d\bm{\theta} = \delta_\gamma &\Longleftrightarrow \int_\Theta \mathcal{L}_{CE}(\bm{\theta})e^{-\beta\cdot\mathcal{L}_{CE}(\bm{\theta})}d\bm{\theta} = \frac{\delta_\gamma}{\alpha} \label{eq:maxwell_boltzmann_const_2}.
\end{align}
\end{theorem}
\begin{proof}
From chain rule, the conditional entropy is written as
\begin{align*}
    H(y|\bm{x};\bm{\theta}) &= H(y,\bm{\theta}|\bm{x}) - H(\bm{\theta}).
\end{align*}
Then, maximizing conditional entropy $H(y|\bm{x};\bm{\theta})$ is equal to minimizing entropy $H(\bm{\theta})$.
In order to minimize the entropy $H(\bm{\theta}) = -\int_\Theta p(\bm{\theta})d(\bm{\theta})$ subject to constraints
\begin{align}
    \int_\Theta p(\bm{\theta})d\bm{\theta} &= 1 \nonumber\\
    \int_\Theta p(\bm{\theta})\mathcal{L}_{CE}(\bm{\theta}) d\bm{\theta} &= \delta_\gamma,
\end{align}
consider the Lagrangian
\begin{align}
    L = -p(\bm{\theta})\ln p(\bm{\theta}) - \beta\mathcal{L}_{CE}(\bm{\theta})p(\bm{\theta}) - \eta p(\bm{\theta}). \label{eq:proof_maxwell_boltzmann_lagrangian}
\end{align}
The Euler-Lagrange equation~\citep{fox1987introduction} for Eq.\ref{eq:proof_maxwell_boltzmann_lagrangian} is
\begin{align}
    \ln p(\bm{\theta}) = \beta\mathcal{L}_{CE}(\bm{\theta}) - \eta + 1, 
\end{align}
with solution
\begin{align}
    p(\bm{\theta}) &= \alpha\cdot e^{-\beta\cdot\mathcal{L}_{CE}(\bm{\theta})}, \label{eq:proof_maxwell_boltzmann_const_1}\\
    \alpha &= e^{1 - \eta}. \label{eq:proof_maxwell_boltzmann_const_2}
\end{align}
From Eq.~\ref{eq:proof_maxwell_boltzmann_const_1} and Eq.~\ref{eq:proof_maxwell_boltzmann_const_2}, we have
\begin{align}
    \int_\Theta e^{-\beta\cdot\mathcal{L}_{CE}(\bm{\theta})}d\bm{\theta}\int_\Theta \mathcal{L}_{CE}(\bm{\theta})e^{-\beta\mathcal{L}_{CE}(\bm{\theta})}d\bm{\theta} = \delta_\gamma. \label{eq:proof_maxwell_boltzmann_const_3}
\end{align}
Consider
\begin{align}
    \Phi(\beta) = \int_\Theta e^{\beta\mathcal{L}_{CE}(-\bm{\theta})}d\bm{\theta}\int_\Theta\mathcal{L}_{CE}(\bm{\theta})e^{-\beta\mathcal{L}_{CE}(\bm{\theta})}d\bm{\theta},
\end{align}
which is an decreasing function of $\beta$.
Since
\begin{align}
    \lim_{\beta\to +\infty}\Phi(\beta) &= 0, \nonumber\\
    \lim_{\beta\to -\infty}\Phi(\beta) &= +\infty,
\end{align}
the continuity of $\Phi(\beta)$ implies that Eq.~\ref{eq:proof_maxwell_boltzmann_const_3} has a solution and this is unique.
Hence, a unique pair $(\alpha,\beta)$ satisfies the constraints of the problem.
Therefore, the Maxwell-Boltzmann distribution is unique.
\end{proof}

This can be regarded as the learning process by focal loss trying to learn a model with parameters that do not deviate too much from the submanifold composed of the family of distributions expressed in Eq.~\ref{eq:maxwell_boltzmann}.
Using Eq.~\ref{eq:maxwell_boltzmann} as the prior of parameter distribution as $\pi = \tilde{p}$, we give the following PAC-Bayes bound by using Thiemann's bound\citep{thiemann2017strongly}.
\begin{corollary}
\label{prop:pac_bayes_maxwell_boltzmann}
Maxwell-Boltzmann distribution in Eq.~\ref{eq:maxwell_boltzmann} induces the following bound.
\begin{align}
    \mathbb{P}\left[\forall{\varsigma}\in P(\Theta), \mathbb{E}_{\bm{\theta}\sim\varsigma}\left[\mathcal{R}(\bm{\theta})\right] \leq J(\zeta, \lambda) + \frac{D_{KL}[\varsigma\|\pi] + \ln\frac{2\sqrt{n}}{\epsilon}}{n\lambda(1-\frac{\lambda}{2})}\right] \leq \epsilon \nonumber\label{eq:pac_bayes_bound}
\end{align}
for some $\epsilon>0$ and $\lambda\geq 0$, where $J(\zeta, \lambda) = \frac{\mathbb{E}_{\bm{\theta}\sim\varsigma}[\hat{\mathcal{R}}(\bm{\theta})]}{1 - \frac{\lambda}{2}}$, $n$ is the sample size of training data and $\mathcal{R}(\bm{\theta}), \hat{\mathcal{R}}(\bm{\theta})$ are expected and empirical error, respectively.
Here, $D_{KL}[\varsigma\|\pi] = \int_{\Theta} \varsigma(\bm{\theta})\ln(\varsigma(\bm{\theta}) / \pi(\bm{\theta}))d\bm{\theta}$ is the Kullback--Leibler divergence (or KL-divergence).
The minimum of the right-hand side of this bound is achieved by
\begin{align}
    \varsigma_\lambda(\bm{\theta}) &= \frac{\pi(\bm{\theta})e^{-\lambda n \hat{\mathcal{R}}(\bm{\theta})}}{\mathbb{E}_{\bm{\theta}\sim\pi}\left[e^{-\lambda n \hat{\mathcal{R}}(\bm{\theta})}\right]},\quad (\text{fixed $\lambda$}),\\
    \lambda &= \frac{2}{\sqrt{\frac{2n \mathbb{E}_\varsigma\left[\hat{\mathcal{R}}(\bm{\theta})\right]}{D_{KL}\left[\varsigma\|\pi\right] + \ln\frac{2\sqrt{n}}{\delta}} + 1} + 1}\quad (\text{fixed $\varsigma$}).
\end{align}
\end{corollary}
This upper bound can be decomposed into an expected error term and a KL-divergence term.
We can see that if the estimators are consistent, the first term approaches a minimum as $n\to\infty$, and the second term also vanishes because the denominator includes the sample size.
On the other hand, for a fixed $n$, the bound can be made tighter by reducing the KL-divergence in the second term.
This induces the following regularizer naturally.
\begin{corollary}
    Focal Loss equips the following regularizer:
    \begin{align}
        D_{KL}[\varsigma\|\pi] &= \int_\Theta \varsigma(\bm{\theta})\ln\frac{\varsigma(\bm{\theta})}{c\cdot\exp\left\{\beta\int_{\mathcal{X}}\sum^{|\mathcal{Y}|}_{y=1}q(\bm{x},y)\ln p(y|\bm{x};\bm{\theta})d\bm{x}\right\}}d\bm{\theta}. \nonumber
    \end{align}
\end{corollary}
KL-divergence for small changes in parameters can be expressed using the Fisher information matrix.
Here, it is known that the space composed of probability distributions is a Riemannian manifold with parameters as the coordinate system, from the framework of information geometry~\citep{amari2000methods,amari2016information}.
On this Riemannian manifold, the Fisher information matrix behaves as a metric, characterizing the curvature of the manifold.
From the above discussion, we can characterize the geometric behavior of focal loss as follows.
\begin{theorem}
    \label{thm:sharpness_focal_loss}
    Learning with the focal loss reduces the local sharpness of the likelihood, if prior $\pi$ and posterior $\varsigma$ are close enough.
\end{theorem}
\begin{proof}
From assumption, let
\begin{align}
    \Delta\xi_i = \xi^\pi_i - \xi^\varsigma_i,
\end{align}
where $\bm{\xi}^\varsigma$ and $\bm{\xi}^\pi$ are parameters of $\varsigma$ and $\pi$.
Consider the quadratic approximation of $d_{kl}(\bm{\xi}) = D_{KL}[\bm{\xi}^\varsigma\|\bm{\xi}]$ as
\begin{align}
    d_{kl}(\bm{\xi}) &= d_{kl}(\bm{\xi}^\varsigma) + \sum_{i}\frac{\partial d_{kl}}{\partial\xi_i}(\xi^\varsigma)\Delta\xi_i \nonumber \\
    &+ \frac{1}{2}\sum_{i,j}\frac{\partial^2 d_{kl}}{\partial\xi_i\xi_j}(\bm{x}^\pi)\Delta\xi_i\Delta\xi_j - o(\|\Delta\bm{\xi}\|^2).
\end{align}
First, since $D_{KL}[p\|q]=0$ if and $p=q$, we have
\begin{align}
    d_{kl}(\bm{\xi}^\varsigma) = D_{KL}[\bm{\xi}^\varsigma\|\bm{\xi}^\varsigma] = 0.
\end{align}
Next, diagonal part of the first variation of the KL-divergence is
\begin{align}
    \frac{\partial}{\partial\xi_i}D_{KL}[\bm{\xi}^\varsigma\|\bm{\xi}] &= \frac{\partial}{\partial\xi_i}\int_{\Theta}p(\bm{\theta};\bm{\xi}^\varsigma)\ln p(\bm{\theta};\bm{\xi}^\varsigma)d\bm{\theta} \nonumber\\
    &- \frac{\partial}{\partial\xi_i}\int_\Theta p(\bm{\theta};\bm{\xi}^\varsigma)\ln p(\bm{\theta};\bm{\xi})d\bm{\theta} \nonumber\\
    &=-\int_\Theta p(\bm{\theta};\bm{\xi}^\varsigma)\frac{\partial}{\partial\xi_i}\ln p(\bm{\theta};\bm{\xi})d\bm{\theta}. \\
    \frac{\partial}{\partial\xi_i}D_{KL}[\bm{\xi}^\varsigma\|\bm{\xi}]_{|_{\bm{\xi}=\bm{\xi}^\varsigma}} &= -\int_\Theta p(\bm{\theta};\bm{\xi}^\varsigma)\frac{\partial}{\partial\xi_i}\ln p(\bm{\theta};\bm{\xi}^\varsigma)d\bm{\theta} \nonumber\\
    &= -\mathbb{E}_{\bm{\xi}^\varsigma}\left[\frac{\partial}{\partial\xi_i}\ln p(\bm{\theta};\bm{\xi}^\varsigma)\right] = 0.
\end{align}
Finally, the diagonal part of the Hessian of the KL-divergence is
\begin{align}
    \frac{\partial^2}{\partial\xi_i\partial\xi_j}D_{KL}[\bm{\xi}^\varsigma\|\bm{\xi}] &= \frac{\partial^2}{\partial\xi_i\xi_j}\int_\Theta p(\bm{\theta};\bm{\xi}^\varsigma)\ln p(\bm{\theta};\bm{\xi}^\varsigma)d\bm{\theta} \nonumber\\
    &- \int_\Theta p(\bm{\theta};\bm{\xi}^\varsigma)\ln p(\bm{\theta};\bm{\xi})d\bm{\theta} \nonumber \\
    &= -\int_\Theta p(\bm{\theta};\bm{\xi}^\varsigma)\frac{\partial^2}{\partial\xi_i\partial\xi_j}\ln p(\bm{\theta};\bm{\xi})d\bm{\theta}. \nonumber\\
    \frac{\partial^2}{\partial\xi_i\partial\xi_j}D_{KL}[\bm{\xi}^\varsigma\|\bm{\xi}]_{|_{\bm{\xi}=\bm{\xi}^\varsigma}} &= -\int_\Theta p(\bm{\theta};\bm{\xi}^\varsigma)\frac{\partial^2}{\partial\xi_i\partial\xi_j}\ln p(\bm{\theta};\bm{\xi})d\bm{\theta} \nonumber \\
    &= -\mathbb{E}_{\bm{\xi}^\varsigma}\left[\frac{\partial^2}{\partial\xi_i\partial\xi_j}\ln p(\bm{\theta};\bm{\xi}^\varsigma)\right] \nonumber\\
    &= \mathfrak{g}_{ij}(\bm{\xi}^\varsigma),
\end{align}
where $I(\bm{\xi})=(\mathcal{g}_{ij}(\bm{\xi}))$ is the Fisher information matrix.
Then, we have
\begin{align}
    D_{KL}[\varsigma\|\pi] = \frac{1}{2}\sum_{i,j}g_{ij}(\bm{\xi}^\varsigma)\Delta\xi_i\Delta\xi_j + o(\|\Delta\bm{\xi}\|^2). \label{eq:kl_fisher}
\end{align}
From Eq.~\ref{eq:kl_fisher} and Eq.~\ref{eq:pac_bayes_bound},
\small
\begin{align}
    \mathbb{P}\left[\forall{\varsigma}\in P(\Theta), \mathbb{E}_{\bm{\theta}\sim\varsigma}\left[\mathcal{R}(\bm{\theta})\right] \leq J(\zeta, \lambda) + \frac{\frac{1}{2}g_{ij}(\bm{\xi}^\varsigma)\Delta\xi^i\Delta\xi^j + \ln\frac{2\sqrt{n}}{\epsilon}}{n\lambda(1-\frac{\lambda}{2})} \right] \leq \epsilon \nonumber
\end{align}
\normalsize
with $o(\|\Delta\bm{\xi}\|^2)$, and the second term of the right-hand side measures the curvature of the log-likelihood.
\end{proof}

\begin{figure*}[ht]
    \centering	
    \includegraphics[width=0.49\linewidth]{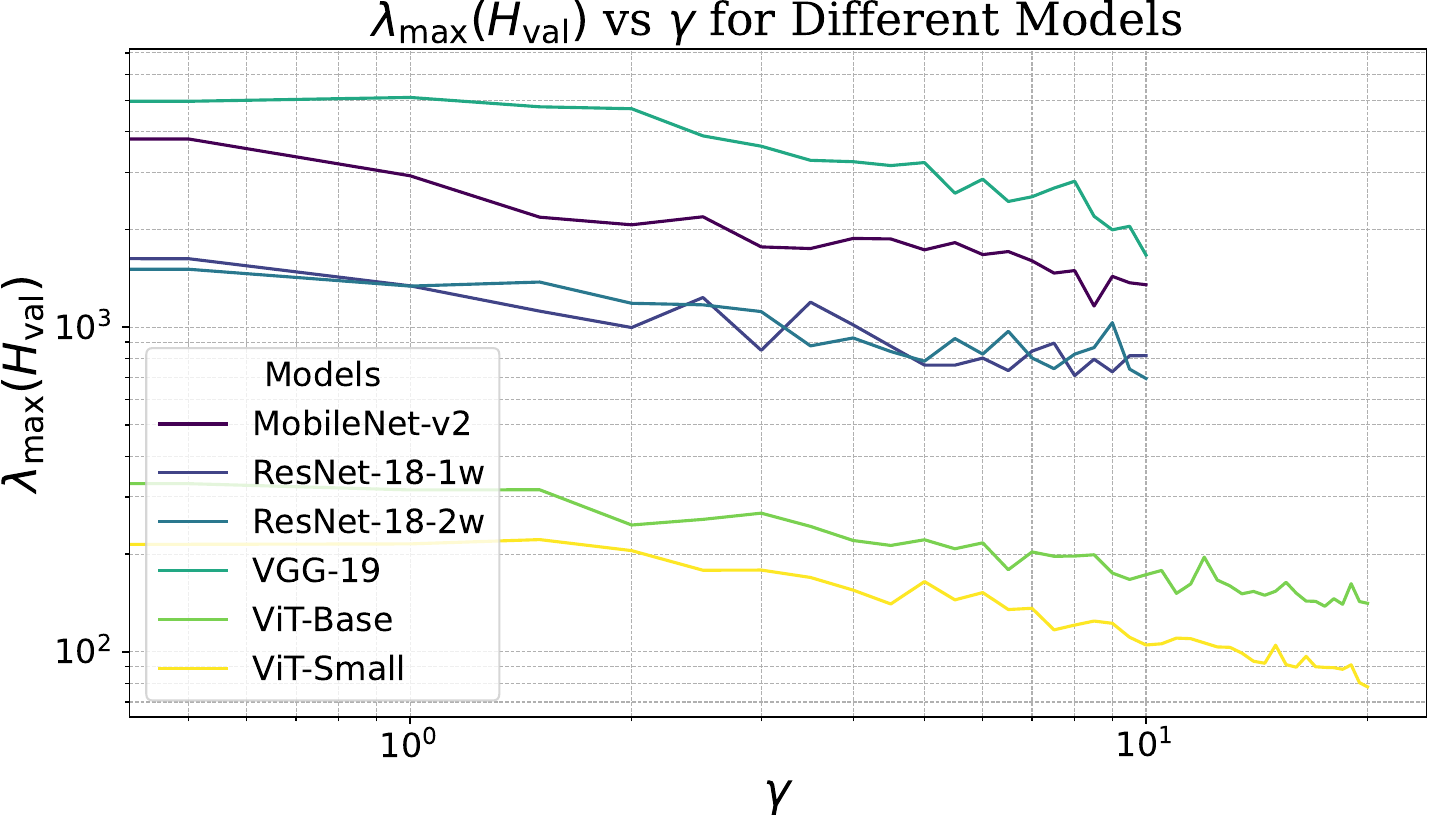}
    \includegraphics[width=0.49\linewidth]{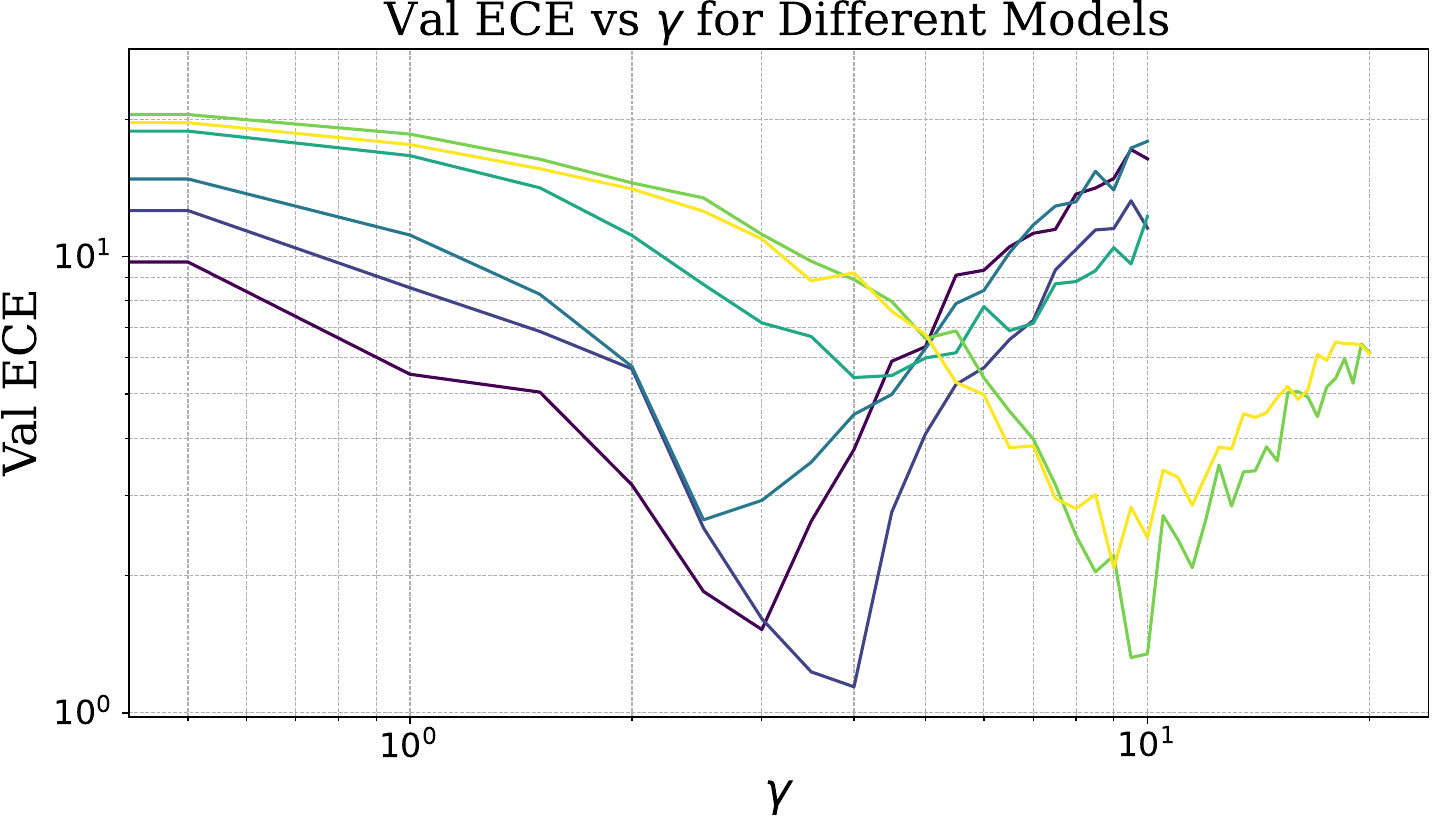}

    \caption{Changes in $\lambda_\text{max}(H_\text{val})$ and expected calibration error (ECE) with respect to hyperparameter $\gamma$ in focal loss training on CIFAR100 using different model architectures. For all architectures, $\lambda_\text{max}(H_\text{val})$ monotonically decreases with increasing $\gamma$, consistent with \cref{thm:sharpness_focal_loss}. ECE reaches its minimum for ViT architecture around $\gamma$ = 10 and for other architectures around $\gamma = 3,4$ . These results confirm the importance of curvature regularization through appropriate focal loss for achieving low ECE.}
    \label{fig:focal_loss_ece}
\end{figure*}

Here, we can also relate focal loss to curvature from another path.
The Taylor expansion of the negative log-likelihood function $\ell(\bm{\theta}) = p_{\bm{\theta}} \coloneqq -\ln p(\cdot;\bm{\theta})$ around the maximum likelihood estimator $\bm{\theta}_0$ is given by
\begin{align}
    \ell(\bm{\theta}) \approx \ell(\bm{\theta}_0) + \frac{1}{2}(\bm{\theta}-\bm{\theta}_0)^\top\frac{\partial^2}{\partial\bm{\theta}^\top\partial\bm{\theta}}\ell(\bm{\theta})\Big|_{\bm{\theta}_0}(\bm{\theta} - \bm{\theta}_0),
\end{align}
and second term corresponds to the curvature around $\bm{\theta}_0$, and is the Fisher information matrix.
In addition, the expansion of $\ell_{FL}(\bm{\theta}) = -(1 - p_{\bm{\theta}})^\gamma \ln p_{\bm{\theta}} \coloneqq -(1 - p(\cdot;\bm{\theta}))^\gamma \ln p(\cdot;\bm{\theta})$ yields the following curvature term:
\begin{align}
    \frac{(1-p_{\bm{\theta}_0})^\gamma}{2}(\bm{\theta} - \bm{\theta}_0)^\top\frac{\partial^2}{\partial\bm{\theta}^\top\partial\bm{\theta}}\ell(\bm{\theta})\Big|_{\bm{\theta_0}}(\bm{\theta} - \bm{\theta}_0). \label{eq:taylor_expansion_of_focal_loss}
\end{align}
Since $(1 - p_{\bm{\theta}}) \geq 0$ and $\gamma \geq 0$, the role of the parameter $\gamma$ is to control the curvature.
In fact, the larger $\gamma$ makes the curvature term smaller.
Theorem~\ref{thm:sharpness_focal_loss} and Eq.~\ref{eq:taylor_expansion_of_focal_loss} both suggest that focal loss behaves as curvature reduction from different approaches.
Therefore, we can see that
\begin{itemize}
    \item focal loss can be regarded as reducing the curvature of the loss function, and
    \item given that existing studies have reported that focal loss is effective for model calibration, it can be expected that curvature reduction is effective for model calibration.
\end{itemize}

The second is to be restated as the following conjecture.
\begin{conjecture}
    \label{cnj:focal_gamma_curvature}
    The Focal loss parameter $\gamma$ controls the curvature and calibration performance.
\end{conjecture}

\subsection{Characterization of Curvature}
There are several approaches to characterizing the curvature of loss surfaces.
The most straightforward are the Hessian $H_{\bm{\theta}}$ of a function and its contractions.
As contractions, traces $\mathrm{Tr}(H_{\bm{\theta}})$ and maximum eigenvalues $\lambda_{\max}(H_{\bm{\theta}})$ are often used.

As another characterization of curvature, the following concept of continuity can be utilized.
\begin{definition}[Lipschitz continuity~\citep{sohrab2003basic}]
    A function $f\colon\mathcal{X}\to\mathbb{R}$ is called Lipschitz continuous on $\mathcal{X}$ with Lipschitz constant $L$ if
    \begin{align}
        \left|f(\bm{u}) - f(\bm{v})\right| \leq L\left\|\bm{u} - \bm{v}\right\| \quad \forall \bm{u},\bm{v}\in\mathcal{X}.
    \end{align}
\end{definition}
In general, if the curvature is small, it means that changes in the neighborhood are relatively small.
This leads to a small Lipschitz constant.

Further, consider the gradient norm $\|\nabla f\|$.
This quantity can be regarded as quantifying the speed of change of the function, which means that the gradient norm is small when the Lipschitz constant is small.
Here, it is important to note the effect of the focal loss parameter $\gamma$ on its gradient.
Let $g(p, \gamma)$ be the gradient of focal loss with $\gamma$, and its explicit form can be written as
\begin{align}
    g(p, \gamma) = (1-p)^\gamma(\gamma p(1-p)^{\gamma - 1} - (1-p)^\gamma \ln p),
\end{align}
where $p$ is the classification probability.
Consider the left-hand limit of $g(p, \gamma)$ for $p$ and $\gamma$ as
\begin{align}
    \lim_{p\to 1^-} g(p, \gamma) = \begin{cases}
    0 & (\gamma = 0\ \text{or}\ 0.5 < \gamma), \\
    0.5 & (\gamma = 0.5), \\
    +\infty & (0 < \gamma < 0.5),
    \end{cases}
\end{align}
and
\begin{align}
    \lim_{\gamma\to 0.5}g(p, \gamma) = 0.5p - \ln p + \gamma \ln p.
\end{align}
This can be confirmed numerically from Fig.~\ref{fig:focal_loss_gradient}.
Thus, it can be seen that $\gamma > 0.5$ is suggested to be desirable.
In our numerical experiments, the hyperparameters are configured according to this suggestion.

\section{Numerical Experiments}
\label{sec:numerical_experiments}
In this section, we design numerical experiments to answer the following questions and confirm our results.
\begin{itemize}
    \item[i)] Does focal loss behave as curvature reduction? (\cref{fig:focal_loss_ece})
    \item[ii)] Can we find a connection between curvature and model calibration performance? (\cref{fig:correlation})
    \item[iii)] Is explicit curvature regularization effectively improving calibration? (\cref{fig:explicit_reg})
\end{itemize}

\subsection{Focal Loss and Curvature Reduction}

\Cref{fig:focal_loss_ece} shows the changes in the maximum eigenvalue of the Hessian, $\lambda_\text{max}(H_\text{val})$ , with respect to the hyperparameter $\gamma$ in focal loss training. Here, $H_\text{val}$ denotes the Hessian of the validation loss. As predicted by our theoretical analysis (Theorem~\ref{thm:sharpness_focal_loss}), $\lambda_\text{max}(H_\text{val})$ consistently decreases across all model architectures as $\gamma$ increases. This confirms that focal loss indeed reduces the curvature of the loss surface.

\subsection{Curvature and Calibration Performance}

\begin{figure}[ht]
    \centering	
    \includegraphics[width=0.5\linewidth]{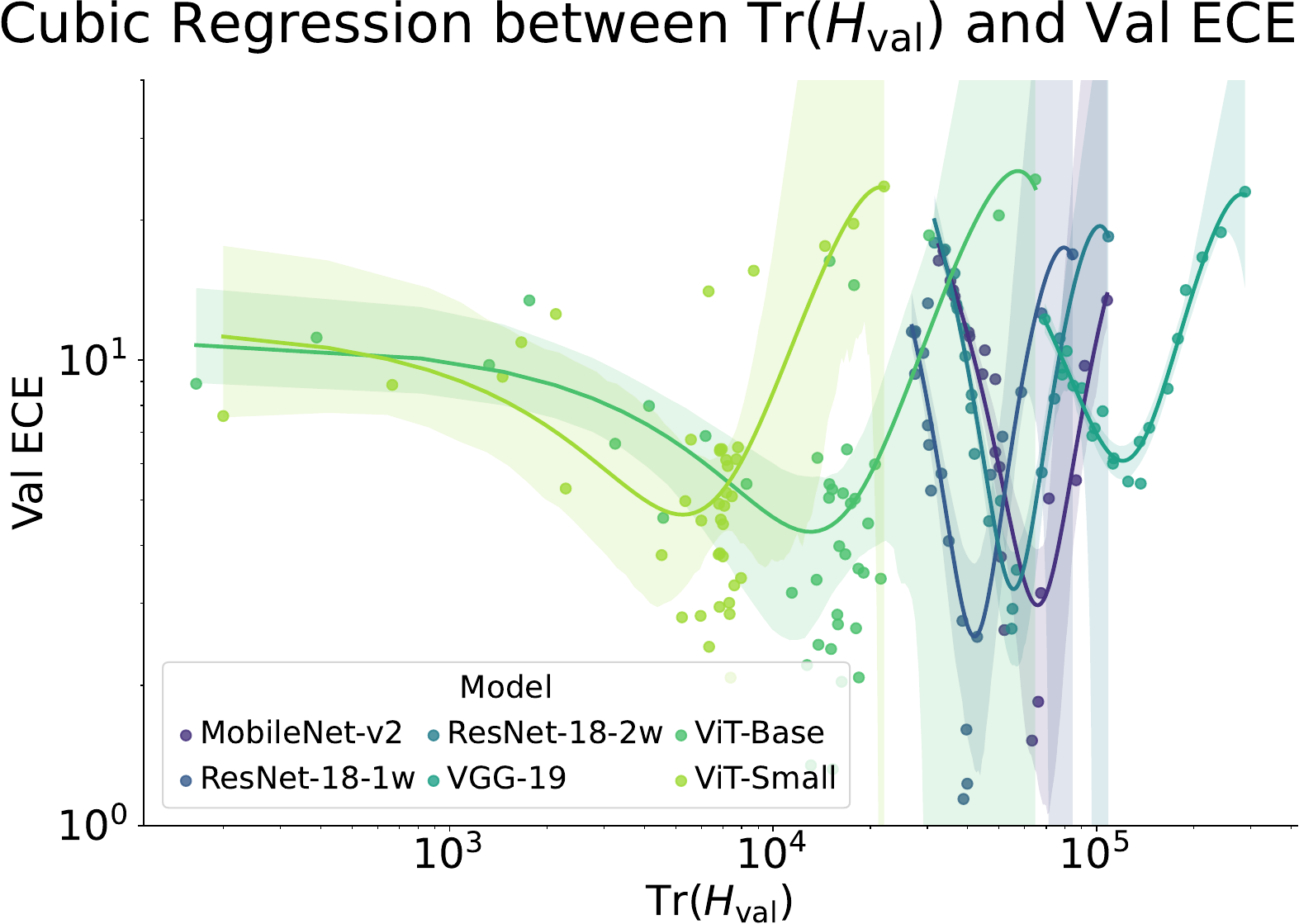}
    
    \caption{Relationship between $\text{Tr}(H_\text{val})$ and expected calibration error (ECE) in focal loss training on CIFAR100 using various DNN architectures. Data points of the same color represent different values of $\gamma$. 
    ECE reaches a minimum peak for all network architectures when $\text{Tr}(H_\text{val})$ is reduced to a certain extent rather than maximizing it. When $\text{Tr}(H_\text{val})$ is too low, training converges to a point significantly different from the convergence point in cross-entropy, resulting in ECE degradation. This indicates that applying appropriate regularization to reduce $\text{Tr}(H_\text{val})$ to a certain extent is crucial for minimizing ECE.}
    \label{fig:correlation}
\end{figure}

\Cref{fig:correlation} explores the relationship between the trace of the Hessian, $\text{Tr}(H_\text{val})$, and the Expected Calibration Error (ECE) during focal loss training on CIFAR-100~\citep{Krizhevsky09learningmultiple} with various DNN architectures\footnote{ResNet18-1w denotes original ResNet18 \citep{he2016deep}, for ResNet18-2w, it denotes doubled width ResNet18.We also use MobileNet-v2~\citep{sandler2018mobilenetv2}, ViT~\citep{dosovitskiy2020image}, and VGG-19~\citep{simonyan2014very}.}. Data points with the same color represent experiments using different values of $\gamma$. The results show that for all network architectures, ECE achieves a minimum value when $\text{Tr}(H_\text{val})$ is reduced to a certain extent, rather than when it's maximized. However, excessively low values of $\text{Tr}(H_\text{val})$ also lead to performance degradation in ECE. This suggests that applying an appropriate degree of regularization to reduce the curvature (reflected by $\text{Tr}(H_\text{val})$) is crucial for achieving optimal calibration performance.

\subsection{Explicit Curvature Regularization}

Focal loss has been demonstrated to improve ECE by reducing curvature to a certain extent. A natural question arises: can we achieve similar improvements in ECE by explicitly minimizing curvature? To address this question, we introduce an experimental protocol based on explicit curvature regularization via Hessian trace minimization.
If explicit regularization of Hessian trace leads to improvement in ECE, then this supports the conjecture that curvature is one of the key factors in model calibration performance.

To implement this approach, we incorporate a regularization term into the loss function that penalizes large values of $\text{Tr}(H)$. Specifically, we introduce a hyperparameter $\tau$ that controls the strength of the regularization:
\begin{align}
    \mathcal{L}_\text{explicit regularization}(\bm{\theta}; \bm{x}, \tau) = \mathcal{L}_{CE}(\bm{\theta}; \bm{x}) + \tau \cdot \text{Tr}(H). \label{eq:explicit_regularization}
\end{align}
At $\tau=0$, Eq.~\ref{eq:explicit_regularization} recovers the original cross entropy loss.

\begin{figure}[tb]
    \centering	
    \includegraphics[width=0.48\linewidth]{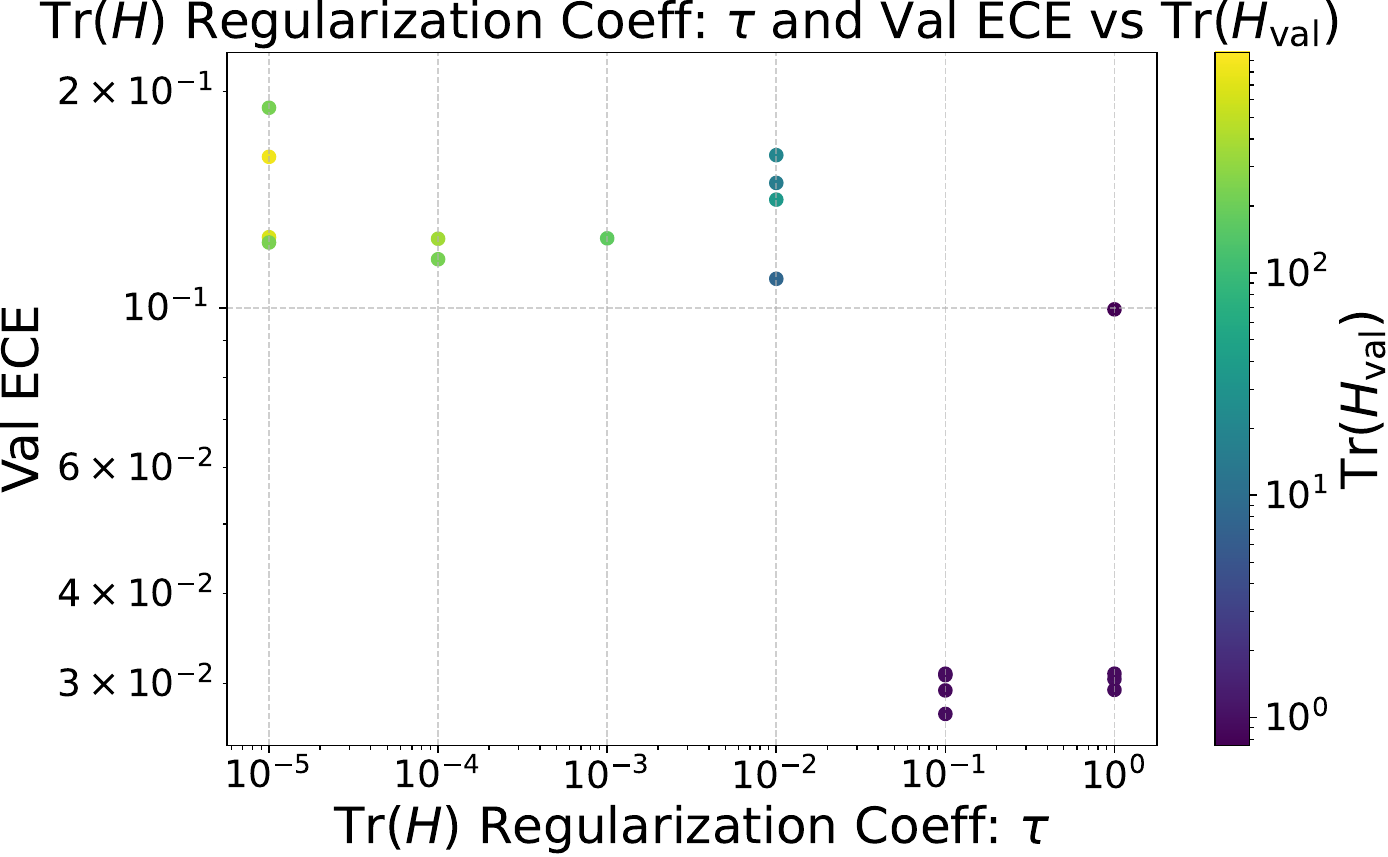}
    \includegraphics[width=0.49\linewidth]{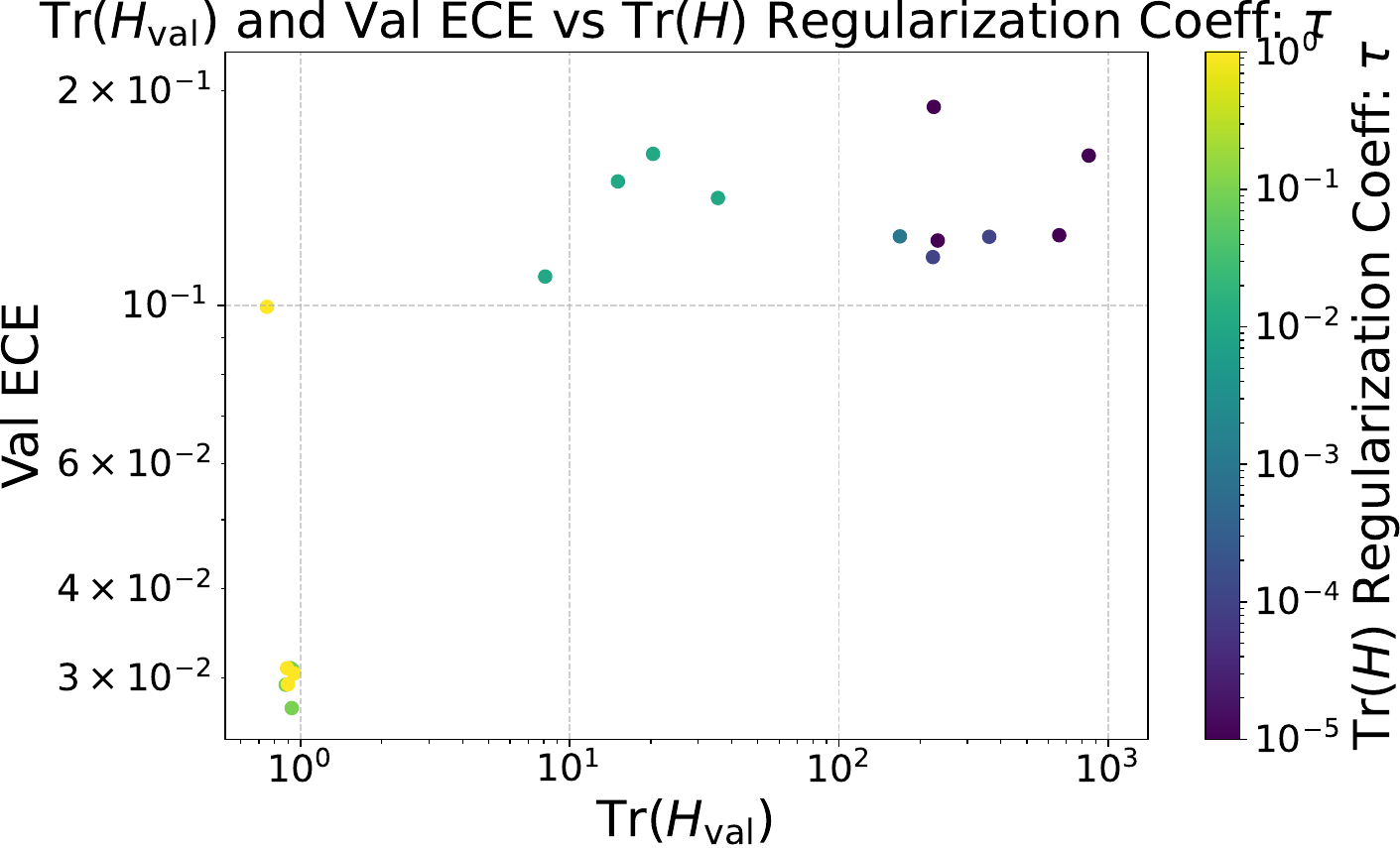}

\caption{This figure presents the results of training a 2-layer MLP with hidden-size 100 on CIFAR10 using a loss function that explicitly incorporates $\text{Tr}(H_\text{val})$ regularization. Each data point corresponds to a different learning rate and $\tau$ value. Higher $\tau$ values resulted in lower ECE and smaller $\text{Tr}(H_\text{val})$, while lower $\tau$ values led to higher ECE and $\text{Tr}(H_\text{val})$.}
\label{fig:explicit_reg}
\end{figure}

While directly computing $\text{Tr}(H)$ at each training step is computationally expensive, we can efficiently approximate it using the Hutchinson method~\citep{hutchinson1989stochastic,hutchinson-method} which utilizes matrix-vector products as
    $\mathrm{Tr}(H) = \mathbb{E}\left[v^\top H v\right] \approx \frac{1}{M}\sum^M_{i=1}v_i^\top H v_i$,
where $M > 0$ and $v\in\mathbb{R}^d$ be a random vector such that $\mathbb{E}[vv^\top] = I$.
This method relies on the power series expansion of the loss function and provides an accurate estimate of $\text{Tr}(H)$ without explicitly calculating the Hessian.

We trained models using the regularized loss function and evaluated their calibration performance using ECE. For each regularization coefficient $\tau$, we employed different learning rates to ensure fair comparison across data points.

Our experiments demonstrate that explicit curvature regularization via Hessian trace regularization effectively improves model calibration. As shown in \cref{fig:explicit_reg}, ECE consistently decreases with increasing $\tau$, indicating a strong correlation between curvature reduction and calibration improvement. This finding highlights the importance of curvature control in achieving well-calibrated models.

\section{Conclusion and Discussion}
\label{sec:conclusion}
In this study, we provided the analysis of focal loss from a geometrical point of view.
First, we reformulated the learning process by focal loss, suggesting that its behavior is like regularization with curvature reduction.
This reformulation allows us to view focal loss as restricting the parameter search to the neighborhood of the submanifold created by the probability distribution family that maximizes entropy.
This consequently leads to equivalence with regularization such that the curvature does not increase during optimization.
We also demonstrated that such a transformation is supported by the asymptotic expansion around the maximum likelihood estimator of focal loss.
Based on these arguments, we then conjecture that the curvature of the function is one of the key factors for model calibration.

To empirically validate these theoretical insights, we conducted a series of numerical experiments to probe the relationships between focal loss, curvature reduction, and model calibration performance. Our findings demonstrate that controlling curvature, implicitly through focal loss or explicitly through direct regularization, significantly enhances model calibration. 
This emphasizes the pivotal role of curvature considerations in achieving well-calibrated machine learning models.

In conclusion, our analysis and experimental results agree that effective curvature control, whether implicit or explicit, is indispensable for optimizing model calibration. 
This study deepens our understanding of focal loss from a geometrical perspective and enriches our understanding of the calibration of deep neural networks.

Finally, we suggest the potential impact of this study by listing the following open questions.
\begin{itemize}
    \item Effect of other calibration algorithms on curvature. This study suggested inductively that focal loss can control curvature and that curvature control leads to improved model calibration. The next question is whether other model calibration algorithms control curvature in the same way.
    \item Differences in calibration behavior that depend on model architecture. Figure~\ref{fig:correlation} shows that the curvature at which calibration performance is best depends on the model architecture. This may imply that the optimal calibration algorithm differs depending on the model architecture.
    \item Relationship to other curvature-aware algorithms. In recent years, various algorithms have shown that performance improvements can be achieved by considering the curvature of the loss surface. These algorithms could similarly lead to improved model calibration performance.
\end{itemize}

\section*{Acknowledgments}
This work was supported by funding from the ZOZO Research.

{
    \small
    \bibliographystyle{ieeenat_fullname}
    \bibliography{main}
}

\end{document}